\documentclass[sigconf]{acmart}

\usepackage{booktabs} 
\setcopyright{none}

\acmDOI{}

\acmISBN{}

\acmConference[]{}{}{}
\acmYear{}
\copyrightyear{}

\acmPrice{}

\protect\newcommand{\set}[1]{\ensuremath{\left\{#1\right\}}}
\newcommand{\prob}[1]{\ensuremath{P\left(#1\right)}}

\newtheorem{problem}{Problem}

\usepackage{xspace}

\newcommand{\astrid}{\textsc{astrid}\xspace}
\newcommand{\goldeneye}{\texttt{Golden\-Eye}\xspace}
\newcommand{\nb}{na\"ive Bayes\xspace}

\newtheorem{hypothesis}{Hypothesis}

\newcommand{\bigoh}{\ensuremath{\mathcal{O}}}

\newcommand{\squishlist}{
 \begin{list}{$\bullet$}
  { \setlength{\itemsep}{0pt}
     \setlength{\parsep}{3pt}
     \setlength{\topsep}{3pt}
     \setlength{\partopsep}{0pt}
     \setlength{\leftmargin}{1.5em}
     \setlength{\labelwidth}{1em}
     \setlength{\labelsep}{0.5em} } }
\newcommand{\squishend}{
  \end{list}  }

\usepackage{subcaption}

\usepackage{booktabs}

\usepackage{calc}

\usepackage{units}

\definecolor{cviolet}{HTML}{57068c}
\definecolor{cgreen}{HTML}{b4be00}
\definecolor{cturquoise}{HTML}{00b0ca}
\definecolor{corange}{HTML}{ff5800}
\definecolor{cpinklight}{HTML}{f3a2c8}
\definecolor{cbluelight}{HTML}{ccd8e4}

\usepackage[linesnumbered,algo2e]{algorithm2e}
\usepackage{tabularx}
\usepackage{colortbl}
\usepackage{dashrule}
\usepackage{array}

\begin{document}
\title{Finding Statistically Significant Attribute Interactions}

\author{Andreas Henelius}
\affiliation{%
  \institution{Finnish Institute of\\ Occupational Health}
\streetaddress{P.O. Box 40}
  \city{00251 Helsinki} 
  \state{Finland}
}
\email{andreas.henelius@ttl.fi}

\author{Kai Puolam{\"a}ki}
\affiliation{%
  \institution{Finnish Institute of\\ Occupational Health}
\streetaddress{P.O. Box 40}
  \city{00251 Helsinki} 
  \state{Finland} 
}
\email{kai.puolamaki@ttl.fi}

\author{Antti Ukkonen}
\affiliation{%
  \institution{Finnish Institute of\\ Occupational Health}
\streetaddress{P.O. Box 40}
  \city{00251 Helsinki} 
  \state{Finland} 
}
\email{antti.ukkonen@ttl.fi}

\renewcommand{\shortauthors}{A. Henelius et al.}

\begin{abstract}
In many data exploration tasks it is meaningful to identify groups of attribute interactions that are specific to a variable of interest. For instance, in a dataset where the attributes are medical markers and the variable of interest (class variable) is binary indicating presence/absence of disease, we would like to know which medical markers interact with respect to the binary class label. These interactions are useful in several practical applications, for example, to gain insight into the structure of the data, in feature selection, and in data anonymisation. We present a novel method, based on statistical significance testing, that can be used to verify if the data set has been created by a given factorised class-conditional joint distribution, where the distribution is parametrised by a partition of its attributes. Furthermore, we provide a method, named \astrid, for automatically finding a partition of attributes describing the distribution that has generated the data. State-of-the-art classifiers are utilised to capture the interactions present in the data by systematically breaking attribute interactions and observing the effect of this breaking on classifier performance. We empirically demonstrate the utility of the proposed method with examples using real and synthetic data.
\end{abstract}

%
%




\maketitle


\section{Introduction}
\label{sec:introduction}
It is often of interest to understand the attribute interactions in a
dataset that are specific to a variable of interest. As an example,
consider a dataset where the attributes are medical markers and
a binary class label indicates presence or absence of a
disease. Here we are interested in determining which medical markers
contribute jointly to the diagnosis and we might, for instance, find
that some markers carry important information on their own, while some
attributes need to be considered jointly.

In this paper we consider attribute interactions in the context of
supervised learning. We say that two or more attributes are
\emph{interacting} if they carry complementary information and are
jointly needed for predicting the class of a data item. Knowledge of
the attribute interactions specific to a variable of interest has
several important real-world applications, e.g., in feature selection
and data anonymisation as we show in this paper. Later, we will give
an exact definition of what we mean by interaction in terms of
conditional probability distributions.

However, finding interacting attributes in a dataset in the general
case is not straightforward and requires complex modelling. In this
paper we instead focus on leveraging classifiers to find interacting
attributes, similarly to \cite{henelius:2014:peek}. We may assume that
state-of-the-art high-performing classifiers must at least implicitly
model and utilise these complex attribute interactions, if they are
able to make accurate predictions. This means that if we can observe
which attributes are jointly used by a classifier for predicting class
labels, we can deduce which attributes are interacting.

In this paper we focus on developing a novel method for finding a
disjoint partition (grouping) $\mathcal{S} = \set{S_1, \ldots, S_k}$
of the attributes of a dataset, such that attributes in the same group
$S_i$ are interacting (dependent) given the class, and attributes in
different groups are independent given the class, respectively.

Given a data matrix $X$, where the rows correspond to data items and
the columns to attributes, respectively, and an associated column vector of
class labels $C$, a classifier tries to model the class probabilities
given the data, i.e., to find $\prob{C \mid X} \propto \prob{X \mid C
} \prob{C}$. Here $P \left( X \mid C \right)$ is the
\emph{class-conditional} distribution of the attributes, which we
focus on here.
Formally, we define the grouping $\mathcal{S}$ to
represent a factorisation of $P \left( X \mid C \right)$ into
independent factors, i.e.,
\begin{equation}\label{eq:P} P
\left( X \mid C ; \mathcal{S}\right) = \prod_{S \in \mathcal{S}}
P\left( X_S \mid C \right ),
\end{equation}where $X_S$ only contains the attributes
in the set $S$. In other words, interacting attributes are in the same
group in $\mathcal{S}$ and, hence, in the same factor in $P \left( X
\mid C ; \mathcal{S}\right)$.

Our method is based on
the following intuition. Assume that we train a classifier $f_1$ using
data from an unfactorised distribution where
all attributes might interact,
and that we train a classifier $f_2$ using
data from a factorised distribution where attribute interactions are
constrained so that only attributes in the same factor may interact. Now,
if the classifiers $f_1$ and $f_2$ cannot be distinguished from each
other in terms of performance, it means that the factorisation
correctly captures the class-dependent structure in the data. On the
other hand, if $f_2$ performs worse than $f_1$ it means that some
essential relationships in the data needed by the classifier are no
longer present, i.e., the factorisation is invalid.

In this paper we consider the two problems of (i) testing whether an
attribute grouping $\mathcal{S}$ correctly captures the attribute
interaction structure, and (ii) automatically finding the
maximum-cardinality attribute grouping $\mathcal{S}$ for a dataset that still satisfies Eq. \eqref{eq:P},
corresponding to the maximally factorised joint class-conditional
distribution.
The first problem is solved using a randomisation test
and for solving the second problem we present the novel \astrid
method.

To accomplish the latter task we also need to introduce a novel
clustering method (Sec.~\ref{sec:clustering}) that can find groupings
when only the reward of the clustering is known but, e.g.,
inter-attribute distances cannot be defined.

By solving these problems we gain insight into \emph{the structure of
  the data} and this has important applications in many domains. We
next consider a few examples demonstrating the practical utility of
attribute interactions.

\vspace*{1em}
\noindent\textbf{Running Example} As a running example we use a
\emph{synthetic dataset} $D$ with 4 attributes $a_1, \ldots, a_4$ and
class labels $C$. The dataset is visualised in
Fig.~\ref{fig:data:synthetic}. It has two classes, each containing 500
data points. The dataset is constructed so that attributes $a_1$ and $a_2$
carry meaningful class information only when considered
jointly. Attribute $a_3$ contains some class information whereas
attribute $a_4$ is random noise. The class-conditional attribute
interaction structure is hence given by the grouping $\mathcal{S} =
\set{\set{1,2}, \set{3}, \set{4}}$. For brevity
we here refer to the attributes by their indices. Attributes in
different groups in $\mathcal{S}$ are independent and constitute the
factors of the class-conditional distribution of the attributes
$P$. We next exemplify the importance of attribute interactions and
their applicability by discussing real-world examples.

\begin{figure*}[t]
\centering
  \includegraphics[width=\linewidth]{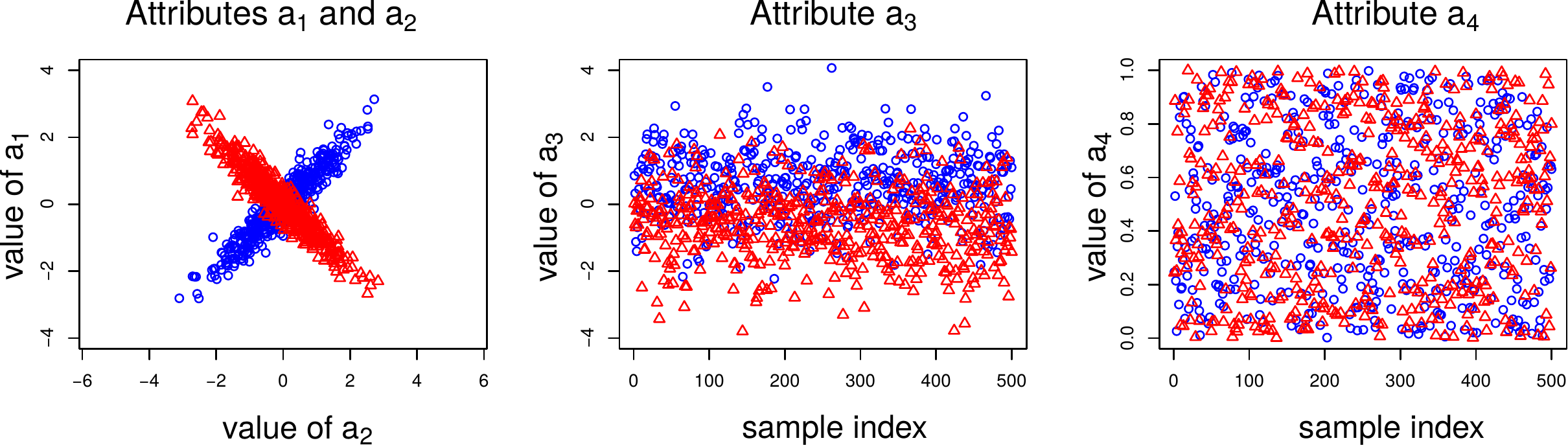}
  \caption{The synthetic dataset used as a running example. Class
    \texttt{0}: blue circles, class \texttt{1}: red triangles.}
\label{fig:data:synthetic}
\end{figure*}

\vspace*{1em}
\noindent\textbf{Example 1.} Attribute interactions have been used,
e.g., in medicine to find attribute combinations that together
constitute risk factors for a procedure \citep{jakulin:2003:b}. In
pharmacovigilance the interactions are important to understand which
drug combinations that can cause adverse drug reactions
\citep{henelius:2015:gepp}. In these cases we want to \emph{identify
  groups of interacting attributes}, i.e., the grouping $\mathcal{S}$.
In Sec.~\ref{sec:results} we show attribute interactions in
a number of real datasets.

\vspace*{1em} 
\noindent\textbf{Example 2.} In several data analysis applications we
need to sample data such that some aspect of the data remains intact,
while the data is otherwise random. E.g., our goal can be to shuffle a
dataset such that the ability of a classifier to accurately make
predictions from the data remains approximately intact. This has
applications in, e.g, creating synthetic datasets for use in model
compression \citep{bucila:2006:a}, or data anonymisation
\citep{AgrawalS00, sweeney:2002:a, BayardoS03}, further exemplified in
Sec.~\ref{sec:results}. As shown in this paper, knowing the attribute
interactions allows us to break attribute interactions in the data
that are not
used by the classifier while the class-conditional joint distribution
remains essentially the same. The important implication of this is
that the classification performance on such a randomised (e.g.,
anonymised) dataset also remains essentially unchanged. We here show
how this can be done with statistical guarantees.

\vspace*{1em}
\noindent\textbf{Example 3.}  An important problem in the analysis of
large datasets is \emph{variable and feature selection} to reduce
training time for models, to reduce the amount of storage space or to
improve classification performance \citep{guyon:2003:a}. If we know
which groups of attributes are interacting with respect to the
classes, it makes sense to perform variable subset selection from
among these groups, thus keeping attribute interactions intact so that
the classifier can exploit them. We demonstrate this in
Sec.~\ref{sec:results}.
\vspace*{0.5em}

\subsection{Contributions}
Our contributions are: \squishlist
\item We show how groups of interacting attributes are related to a
  factorisation of the conditional data distribution.
\item We present and study the two problems of (i) assessing whether a
  particular grouping of attributes represents the class-conditional
  structure of a dataset using statistical hypothesis testing
  (Sec. \ref{sec:1group}) and (ii) automatically discovering the
  attribute grouping of highest granularity, with \textsc{a}utomatic
  \textsc{str}ucture \textsc{id}entification (\astrid)
  (Sec. \ref{sec:automatically}).
\item We present a novel polynomial time clustering algorithm that relies on
  a certain monotonic, but in practice very intuitive score
  function (Sec. \ref{sec:clustering}).
\item We present an experimental verification and discuss several
  real-world data analysis scenarios that become possible through
  knowledge of the class-conditional joint data distribution
  (Secs. \ref{sec:experiments}--\ref{sec:discussion}).  \squishend

\section{Related Work}
\label{sec:relatedwork}
This work has been motivated and influenced by the recently introduced
\goldeneye algorithm \cite{henelius:2014:peek,henelius:2015:gepp}, as
well as the results about classifiers and attribute interactions in
\cite{ojala2010jmlr}.  \cite{henelius:2014:peek} and the method
introduced in this paper share the same randomisation scheme, but
ultimately address different problems: \cite{henelius:2014:peek}
reveals the structure of a classifier function (even the structure
imposed by overlearning), while the approach presented here reveals
only the structure really present in the data {\em and} that can be
modelled by training the classifier. As opposed to the heuristic used
in \cite{henelius:2014:peek}, we approach the problem in a principled
way using statistical significance testing, constituting a
considerable improvement over the previous method.

The main result in this paper is a randomisation test for testing the
hypothesis that an observed dataset has been sampled from a given
class-conditional joint distribution with a particular factorised
form. The problem considered here is closely related to the
permutation test in \cite{ojala2010jmlr}, but instead of only
considering a fully factorised class-conditional distribution (i.e.,
the one assumed by the \nb classifier), the test considered here is
valid for any given factorisation. As a consequence our method can be
used to reveal the attribute interactions present in the data in terms
of the class-conditional joint distribution. Our second contribution
is an algorithm for automatically determining the structure of the
data in terms of finding the factorisation with the highest
granularity such that a classifier trained using the factorised data
is indistinguishable from a classifier trained using the original
data.

Moreover, various methods to study attribute interactions in general
have been proposed, see, e.g., \cite{freitas:2001:a} for a review on
the topic in data mining. Applications to feature selection have been
studied in, e.g., \cite{zhao:2007:a, zhao:2009:a}. \cite{tatti:2011:a}
investigated finding maximally dependent successively ordered
attributes while \cite{mampaey:2013:a} clustered correlated attributes
into groups. \cite{jakulin:2003:a} proposed a method for quantifying
the degree of interaction and \cite{jakulin:2004:a} consider
factorising the joint data distribution and presented a method for
testing the significance of the found attribute interactions
(experiments limited to two and three-way interactions). Investigating
the structure of the data in terms of factorising the joint
distribution is also the topic of Bayesian network learning (e.g.,
\cite{koski:2012:a}).

\section{Verifying a Single Grouping}
\label{sec:1group}
We first introduce the necessary notation, after which we present a
hypothesis testing framework for studying a given attribute grouping.

\subsection{Preliminaries}
Let $X$ be an $n \times m$ data matrix, where $X(i,\cdot)$ denotes the
$i$th row (item), $X(\cdot,j)$ the $j$th column (attribute) of $X$,
and $X(\cdot,S)$ the columns of $X$ given by $S$, where $S\subseteq
[m]=\{1,\ldots,m\}$, respectively.  Let $\mathcal{C}$ be a finite set
of class labels and let $C$ be an $n$-vector of class labels, such
that $C\left(i\right)$ gives the class label for $X(i,\cdot)$. We
denote a dataset $D$ by the tuple $D=\left(X,C\right)$ and the set of
all possible datasets by $\mathcal{D}$.

We denote by $\mathcal{P}$ the set of disjoint partitions of
$[m]=\{1,\ldots,m\}$, where a partition $\mathcal{S}\in\mathcal{P}$
satisfies $\cup_{S\in\mathcal{S}}{S}=[m]$ and for all $S,S'\in\mathcal{S}$
either $S=S'$ or $S\cap S'=\emptyset$, respectively.

Here we assume that the dataset has been sampled i.i.d., i.e., the dataset
$D$ follows a joint probability distribution given by
\begin{equation}
  \label{eq:jdfull}
  \begin{array}{lcl}
\prob{D} &=&
\prod_{i\in[n]}{P(X\left(i,\cdot\right),C\left(i\right))}\\
&=&\overbrace{\prod_{i \in \left\lbrack n \right\rbrack} \prob{X\left(i,\cdot \right) \mid C(i)}}^{\prob{X  \mid C}} \prob{C \left( i \right)},
  \end{array}
\end{equation}
where $\prob{X\mid C}$ is the \emph{class-conditional
  distribution}. We consider a factorisation of $\prob{D}$ into
class-conditional factors given by the grouping $\mathcal{S}\in\mathcal{P}$ and write
\begin{equation}
\label{eq:jdfact}
  \prob{D} = \overbrace{\prod_{i \in \left\lbrack n \right\rbrack} \prod_{S\in \mathcal{S}} \prob{X \left(i, S \right) \mid C \left( i \right)}}^{\prod_{S\in\mathcal{S}}{\prob{X(\cdot,S)  \mid C}}} \prob{C \left( i \right)} .
\end{equation}
Given an observed dataset $D_0 \in \mathcal{D}$, we want to
investigate the structure of the data in terms of groupings
$\mathcal{S}\in\mathcal{P}$ and a natural approach is to formulate a null
hypothesis:
\begin{hypothesis}
\label{hyp:null}
  The observed dataset $D_0$ has been sampled from a distribution
  given by Eq.~\eqref{eq:jdfact} with the grouping given by
  $\mathcal{S}\in\mathcal{P}$.
\end{hypothesis}
We now devise a framework to test this hypothesis.

\subsection{Hypothesis Testing Framework}
Hypothesis~\ref{hyp:null} states that the dataset $D_0$ has been
sampled from a distribution that follows the form given by
Eq.~\eqref{eq:jdfact} with the groups given by $\mathcal{S}$. This
hypothesis can be evaluated empirically using a randomisation test,
for which we need (i) a test statistic and (ii) the distribution of
the test statistic under the null hypothesis. The value of the test
statistic for the observed data is compared to the distribution of the
test statistic under the null hypothesis. The outcome of the
comparison is typically reported in the form of a $p$-value denoting
the probability of obtaining a result at least as extreme as the
observed one under the null hypothesis. Inference regarding the
hypothesis is then made at a significance level $\alpha$ denoting the
probability of a Type I error.

\subsubsection{Test Statistic}
Assume for now that the test statistic yields a real number for each
dataset in $\mathcal{D}$, i.e., $T : \mathcal{D} \mapsto
\mathbb{R}$. The exact form of the test statistic we use is
described in detail later in Sec. \ref{sec:class} after we have
presented the general framework.

\subsubsection{\goldeneye Permutation}

In this paper we sample datasets using the \goldeneye permutation,
first described by Henelius et al.
\cite{henelius:2014:peek}. The \goldeneye permutation is parametrised
by a grouping $\mathcal{S}\in\mathcal{P}$ and creates a sample from
the set of all datasets by permuting the columns of the original data
within-class so that columns in the same group $S\in\mathcal{S}$ are
permuted together.

More formally, a new permuted dataset
$D^\mathcal{S}=\left(X^\mathcal{S},C\right)$ is created by permuting
the data matrix of the dataset $D_0=\left(X_0,C\right)$ at random. The
permutation is defined by $m$ bijective permutation functions $\pi_j:
[n]\mapsto [n]$ sampled uniformly at random from the set of allowed
permutations functions. The new data matrix is then given by
$X^\mathcal{S}\left(i,j\right)=X_0\left(\pi_j\left(i\right),j\right)$. The
allowed permutation functions satisfy the following constraints for all
$i\in [n]$, $j,j'\in[m]$, and $S\in\mathcal{S}$:
\begin{enumerate}
\item permutations are within-a class, i.e.,
  $C\left(i\right)=C\left(\pi_j\left(i\right)\right)$, and
  \item items within a group are permuted together, i.e., $j\in
    S\wedge j'\in S\implies
    \pi_j\left(i\right)=\pi_{j'}\left(i\right)$.
\end{enumerate}

Let $\mathcal{D}_\mathcal{S} \subseteq \mathcal{D}$ be the set of
datasets that can be generated using the \goldeneye permutation with
the grouping $\mathcal{S}$. We make the following two observations.
\begin{lemma}\label{lem:g1}
  Each invocation of the \goldeneye permutation produces each of the
  datasets in $\mathcal{D}_\mathcal{S}$ with uniform probability.
\end{lemma}
\begin{lemma}\label{lem:g2}
  The datasets in $\mathcal{D}_\mathcal{S}$ have equal probability under the
  distribution of Eq. \eqref{eq:jdfact}, parametrised by $\mathcal{S}$.
\end{lemma}
\begin{proof}
  The proofs follow directly from the definition of the permutation
  and the probability distribution of Eq. \eqref{eq:jdfact}.
\end{proof}

It follows from Lemmas~\ref{lem:g1} and \ref{lem:g2} that if the
original dataset $D_0$ is from the distribution defined by
Eq.~\eqref{eq:jdfact} then it is exchangeable with the datasets
sampled using \goldeneye, hence, an empirical $p$-value defined as
follows is valid, i.e., it is stochastically larger than the unit
distribution in $[0,1]$ under the null hypothesis that the data
originates from the distribution given by Eq.~\eqref{eq:jdfact},
\begin{equation}
\label{eq:empiricalp}
p_\mathcal{S} = \frac{1 + \sum_{i=1}^R I \left\lbrack T\left( D_i^\mathcal{S} \right) \geq T
  \left( D_0\right) \right\rbrack}{1 + R},
\end{equation}
where $I\left\lbrack \Box \right\rbrack$ is the indicator function
which equals unity if $\Box$ is true and zero otherwise,
$T:\mathcal{D}\mapsto{\mathbb{R}}$ is the test statistic, $D_0$ is the
original observed dataset, and $D_i^\mathcal{S} \in
\mathcal{D}_\mathcal{S}$, where $i\in [R]$, are $R$ samples created by
the \goldeneye permutation parametrised by $\mathcal{S}$. If we here
obtain $p_\mathcal{S}<\alpha$, we can reject the null hypothesis that
the dataset $D_0$ could have been generated by the distribution given
by Eq.~\eqref{eq:jdfact} with a significance level $\alpha$. We next
specify the exact form of the test statistic $T$.

\subsection{Hypothesis Testing Using Classifiers}
\label{sec:class}

The above described framework is valid for any test statistic $T$, but
a poor choice of $T$ could result in an unnecessarily large number of
Type II errors, i.e., failure to detect an interaction of
attributes. The test statistic should reflect how well $\mathcal{S}$
captures the structure of the data. As discussed in
Sec. \ref{sec:introduction}, it is reasonable to assume that
high-performing classifiers internally model the class-conditional
joint distribution of the attributes (Eq.~\eqref{eq:jdfact}). A
classifier is a function that tries to predict classes given a row
from a data matrix and is typically generated using a training dataset
containing both the data matrix and the class vector
$D=\left(X,C\right)$. We denote a classifier trained using the dataset
$D$ by $f_D:\mathcal{X}\mapsto\mathcal{C}$, where $\mathcal{X}$
denotes the set of all possible rows of the data matrix. Further
assume that we have a separate independent test dataset from the same
distribution as $D_0$, denoted by
$D_\mathrm{test}=\left(X_\mathrm{test},C_\mathrm{test}\right)$. With
these components, we define the test statistic as
\begin{definition}\emph{Test Statistic}
  \label{def:teststatistic}
  Given the above definitions, the test statistic for a dataset
  $D\in\mathcal{D}$ is given by
  \begin{equation}
    \label{eq:teststatistic}
    T\left(D\right)=\frac 1{n_\mathrm{test}}
    \sum_{i=1}^{n_\mathrm{test}}{ I\left\lbrack
      f_D\left(X_\mathrm{test}\left(i,\cdot\right)\right)=C_\mathrm{test}\left(i\right)
      \right\rbrack},
  \end{equation}
  where $n_\mathrm{test}$ is the number of items in the test dataset.
\end{definition}
We have here chosen, for simplicity to use accuracy, but other
performance metrics could be used as well, e.g., the $F_1$ measure.

Finally, we cast the above presented hypothesis testing procedure in
the form of a problem:
\begin{problem}
\label{prob:structuretest}
Given an observed dataset $D_0 \in \mathcal{D}$, a grouping
$\mathcal{S}$, and a classifier $f$, determine at a level $\alpha \in
\lbrack 0, 1\rbrack$ if $D_0$ has
been sampled from a distribution given by Eq.~\eqref{eq:jdfact}
such that the factors are given by $\mathcal{S}$.
\end{problem}
To solve Prob.~\ref{prob:structuretest} we proceed as follows: (i) use
the test statistic of Definition~\ref{def:teststatistic} with the
classifier $f$, (ii) determine the distribution of the test statistic
under the null hypothesis from datasets generated from the observed
dataset using the \goldeneye permutation parametrised by the grouping
$\mathcal{S}$, and finally (iii) compute the $p$-value in
Eq.~\eqref{eq:empiricalp}, after which we evaluate the null hypothesis
at the significance level $\alpha$. If we find
$p_\mathcal{S}\geq\alpha$ we conclude that the null hypothesis cannot
be rejected and it cannot be ruled out that $D_0$ originates from a
distribution given by Eq.~\eqref{eq:jdfact} with groups given by
$\mathcal{S}$.

\subsubsection{About the Statistical Significance Testing Formulation}
We have formulated our problem as a statistical hypothesis testing
problem. The null hypothesis is that the observed data originates from
a distribution that is of the form defined by
Eq. \eqref{eq:jdfact}. If we obtain a $p$-value less than $\alpha$ we can
reject the null hypothesis at the given significance level; this means
that we have some evidence of the fact that the attributes actually
exhibits some interactions that were broken by the given
grouping. However, in the opposite case when the $p$-value stays above
$\alpha$, we can only conclude that we cannot reject the null
hypothesis, but we cannot necessarily conclude that the null hypothesis is
true. As a simple example, if we use the \nb classifier as our
classifier function then any grouping will typically obtain a high
$p$-value (i.e., we cannot reject the null hypothesis), as shown
later, e.g., in Tab.~\ref{res:tab:synthetic:nb}, independent of the
actual structure present in the data; this behavior results simply
from the fact that the \nb classifier models data by a conditionally
independent distribution and hence, it provides a poor test statistic
for our purposes. Therefore, if we try to find a maximum cardinality
grouping for which the null hypothesis cannot be rejected, then we tend
to err (Type II error) towards a grouping of higher cardinality if the
test statistic (here constructed by using a classifier) fails to capture the structure present in the data.


\section{Automatically Finding Groupings (\astrid)}
\label{sec:automatically}

The above described method allows us to test whether a
\emph{particular grouping} $\mathcal{S}$ describes the structure of
the data in terms of the factorisation in Eq.~\eqref{eq:jdfact}. The
following problem is a natural extension of the above discussion:
\begin{problem}
  \label{prob:structureid}
Given an observed dataset $D_0$ and a classifier $f$, find the
grouping $\mathcal{S}$ of cardinality $k$ such that the accuracy is
maximised.
\end{problem}

Instead of specifying the cardinality $k$ in
Prob.~\ref{prob:structureid} in advance, we can also use the
confidence level $\alpha$ as a heuristic for model selection as
follows. We first find the optimal partitions for $k = 1, \ldots, m$,
after which we compute the $p$-value using Eq.~\eqref{eq:empiricalp}
for all the $m$ partitions. The desired solution is given by the
highest-cardinality grouping satisfying $p_\mathcal{S}\geq\alpha$.

Interestingly, Prob.~\ref{prob:structureid} can be viewed as an
instance of a generic clustering problem:
\begin{problem}\label{prob:clustering}
  Given integers $k$ and $m$ and a reward function $\hat
  T:\mathcal{P}\mapsto{\mathbb{R}}$, where $\mathcal{P}$ is the set of
  all partitions of $[m]$, find a partition
  $\mathcal{S}\in\mathcal{P}$ of size $|\mathcal{S}|=k$ such that the
  reward $\hat T(\mathcal{S})$ is maximised.
\end{problem}
Prob.~\ref{prob:structureid} reduces to Prob.~\ref{prob:clustering} if
we use the expected accuracy $T$ defined in
Eq.~\eqref{eq:teststatistic} as the the reward function $\hat T$ in
the clustering problem;
\begin{equation}\label{eq:that}
  \hat T(\mathcal{S})=\frac 1{R'}\sum_{i=1}^{R'}T\left(D^\mathcal{S}_i\right),
\end{equation}
where $D^\mathcal{S}_i$, $i\in [R']$, is a random dataset produced by
the \goldeneye permutation parametrised by $\mathcal{S}$. We provide a
polynomial time heuristic algorithm to solve
Prob. \ref{prob:clustering}. The algorithm yields the exact solution
to the problem if the accuracy is monotonic (Def.~\ref{def:mono} and
Theorem \ref{thm:optimal} below). The clustering algorithm and its
properties are described in more detail later in
Sec.~\ref{sec:clustering}.

The monotonicity of the accuracy means that if $\mathcal{S}_0$ is the
solution to Prob. \ref{prob:structureid}, then the accuracy $\hat
T(\mathcal{S})$ is reduced if any group in $\mathcal{S}_0$ has been
broken in $\mathcal{S}$. More specifically, let
$\mathcal{S}=\{S_1,\ldots,S_k\}$ be a partition of $[m]$, and let
$\mathcal{S}'=\{S_{1a},S_{1b},S_2,\ldots,S_k\}$, where $S_1=S_{1a}\cup
S_{1b}$ and $S_{1a}\cap S_{1b}=\emptyset$. Assume there exist a group
$Q\in\mathcal{S}_0$ such that $Q\cap S_1=Q$, but $Q\cap S_{1a}\ne Q$
and $Q\cap S_{1b}\ne Q$, i.e., splitting $S_1$ into $S_{1a}$ and
$S_{1b}$ has broken at least the group $Q$ in $S_0$. The monotonicity
of the accuracy implies here that $\hat
T\left(\mathcal{S}'\right)<\hat T\left(\mathcal{S}\right)$, i.e.,
breaking the interactions in $Q$ means that the classifier cannot
utilise them fully, which is expected to reduce classification
accuracy. If the monotonicity is preserved to a sufficient accuracy,
then we expect that the clustering algorithm can efficiently and
accurately provide a solution to Prob.~\ref{prob:structureid}.

Solving Prob.~\ref{prob:structureid} for all $k$ requires evaluation
of the accuracy $\hat T(\mathcal{S})$ for $\bigoh\left(m^2\right)$
different values of $\mathcal{S}$. If we want to further find the
highest-cardinality grouping satisfying $p_\mathcal{S}\ge\alpha$ then
$\bigoh\left(m\right)$ $p$-value computations are additionally needed
which, however, does not increase the computational complexity
compared to just solving Prob.~\ref{prob:structureid}.

\subsection{Clustering Problem}
\label{sec:clustering}
As discussed in Sec.~\ref{sec:automatically}, to find an optimal
grouping we need to solve a generic clustering problem
(Prob.~\ref{prob:clustering}).

If no assumptions of the reward function are made then verifying that
a given partition is a solution to Prob.~\ref{prob:clustering}
requires evaluation of $\hat T\left(\mathcal{S}\right)$ for all
partitions $S\in\mathcal{P}$ of size $k$. In order to provide a polynomial
time algorithm to solve the problem we hence need to make some
assumptions regarding the form of the reward function. We can indeed
devise an efficient and exact algorithm if $\hat T$ behaves
consistently in the sense that the reward function decreases if any of
the clusters in the (a priori unknown) solution to
Prob.~\ref{prob:clustering} are broken. We call this property {\em
  monotonicity} and define it formally as follows.
\begin{definition}\label{def:mono}
  For given integers $k$ and $m$ and a reward function $\hat T$, as
  defined above, let $\mathcal{S}_0$ be the solution to
  Prob.~\ref{prob:clustering}. The reward function $\hat T$ is {\em
    monotonic} iff all $\mathcal{S},\mathcal{S}'\in\mathcal{P}$
  satisfying $F\left(\mathcal{S}\right)\subset
  F\left(\mathcal{S}'\right)$ also satisfy $\hat
  T\left(\mathcal{S}\right)<\hat T\left(\mathcal{S}'\right)$, where we
  have used $F(\mathcal{S})=\left\{X\in\mathcal{S}_0\mid\exists
  Y\in\mathcal{S}\ldotp X\subseteq Y\right\}$.\footnote{If for example
    $\mathcal{S}_0=\{\{1,2\},\{3,4\}\}$ then
    $F\left(\{\{1,2,3\},\{4\}\}\right)=\{\{1,2\}\}$, i.e., the
    function $F\left(\mathcal{S}\right)$ returns members of
    $\mathcal{S}_0$ that are unbroken in $\mathcal{S}$.}

\end{definition}
We propose a heuristic clustering algorithm, described in
Alg.~\ref{listing:sid}, for solving Prob.~\ref{prob:clustering}. The
algorithm first sorts the attributes by iteratively moving attributes
to singleton clusters such that the accuracy is maximised at each step
(lines 1--6). After this the algorithm finds the $k$-segmentation of
the ordered set of attributes (lines 7--9) corresponding to the
solution of Prob.~\ref{prob:clustering} (lines 10--11). Even though
the algorithm is heuristic in the general case, it provides an exact
solution if the reward function is monotonic.
\begin{theorem}
  \label{thm:optimal}
  Alg. \ref{listing:sid} solves Prob.~\ref{prob:clustering} exactly if
  the reward function $\hat T$ is monotonic.
\end{theorem}
\begin{proof}
Denote by $\mathcal{S}_0$ the solution to
Prob.~\ref{prob:clustering}. In the sorting phase of
Alg.~\ref{listing:sid} (lines 1--6) the attributes are ordered into a
vector $a_1,\ldots,a_m$ so that all clusters $S\in \mathcal{S}_0$
appear in a continuous segment. This follows directly from the
monotonicity: consider an iteration of this loop (lines 3--5). Let
$S\in\mathcal{S}_0$ be the cluster that contains $j$ obtained in line
3, i.e., $j\in S$. If we have not yet processed all attributes in $S$,
i.e., if after line 4 $S\cap C\ne\emptyset$, then in the next
iteration of the loop we must choose a value of $j$ from $S\cap C$,
because from monotonicity it follows that choosing $j$ not in $S\cap
C$ would result in a lower reward than choosing $j$ from $S\cap C$.

Therefore, because the vector $a_1,\ldots,a_m$ contains the clusters
in $\mathcal{S}_0$ in continuous segments, the problem reduces to finding
$k-1$ segment boundaries that split the vector into $k$ segments, with
each segment corresponding to a cluster in $\mathcal{S}_0$. We can make
an observation that if we split the attributes in $[m]$ into two
clusters $L\subseteq [m]$ and $[m]\setminus L$, then the reward $\hat
T\left(\{L,[m]\setminus L\}\right)$ is larger if the split into two
cluster does not break clusters in $\mathcal{S}_0$. In other words, $\hat
T\left(\{L,[m]\setminus L\}\right)>\hat T\left(\{L',[m]\setminus
L'\}\right)$ if there is a subset of clusters $R\subseteq \mathcal{S}_0$
such that $L=\cup_{X\in R}{X}$ and there is no subset of clusters
$R'\subseteq\mathcal{S}_0$ such that $L'=\cup_{X\in R'}{X}$. Therefore,
it suffices to compute the costs of all segmentations into two (line
7), and pick the $k-1$ segment boundaries corresponding to the highest
rewards (line 9). The resulting segments, defined in line 10, must
then correspond to the clusters in $\mathcal{S}_0$.
\end{proof}
It should be noted that the reward function we use here is not
monotonic, but the assumption of monotonicity appears to hold well in
practice. This can be compared, e.g, to the assumption of normality of
data: even though the assumption does not hold exactly it is still a
useful approximation.

Notice that Alg.~\ref{listing:sid} can also be used to efficiently
give the clustering for all values of $k$ since lines 1--7 are common
for all values of $k$ and only lines 8--11 need to be re-run for
different values of $k$. The clustering algorithm therefore requires
only $\bigoh\left(m^2\right)$ evaluations of the reward function to
find clusterings for all values of $k\in[m]$.

Moreover, usually clustering cost functions are defined in terms of
distances between cluster centroids or data points. In our case we do
not, however, have any well-defined distances between data points and,
hence, normal clustering algorithms are not applicable. Instead, we
have to find the correct grouping based on the value of the reward
(cost) function alone, which makes the problem more challenging.
However, the monotonicity assumption of Def.~\ref{def:mono} allows us,
in fact, to find optimal solutions in polynomial time. To the best of
our knowledge, this particular approach to finding clusterings has not
been considered previously, and we think it may have applications also
in other contexts.

\begin{algorithm2e}[t!]
  \SetKwInOut{Input}{input}
  \SetKwInOut{Output}{output}
  \Input{
    $k$, $m$, and $\hat T$ as defined in Prob. \ref{prob:clustering}}
  \Output{partition $\mathcal{S}$ of size $k$ maximising $\hat T(\mathcal{S})$}
  \tcc{Sorting}
  Let $B\leftarrow\emptyset$, $C\leftarrow [m]$\;
  \For{$i=1$ to $m$}{
    Let $j\leftarrow\arg\max_{j\in C}\hat T(\{\{C\setminus\{j\}\}\cup\cup_{l\in B\cup\{j\}}\{\{l\}\})$\;
    Let $B\leftarrow B\cup\{j\}$, $C\leftarrow C\setminus\{j\}$\;
    Let $a_i\leftarrow j$\;
  }
  \tcc{Grouping}
  Let $t_i\leftarrow \hat T(\{\{a_1,\ldots,a_i\},\{a_{i+1},\ldots,a_m\}\})$ for all $i\in[m-1]$\;
  Let $i_0\leftarrow 0$, $i_k\leftarrow m$\;
  Let $i_1<\ldots<i_{k-1}$ be such that $\{t_{i_1},\ldots,t_{i_{k-1}}\}$ are the $k-1$ largest values of $\{t_1,\ldots,t_{m-1}\}$\; 
  Let $S_j\leftarrow\{a_{i_{j-1}+1},\ldots,a_{i_j}\}$ for all $j\in[k]$\;
    Let $\mathcal{S}\leftarrow\{S_1,\ldots,S_k\}$\;
  \Return{$\mathcal{S}$}
  \caption{\label{listing:sid}
The clustering algorithm.}
\end{algorithm2e}

\section{Experiments}
\label{sec:experiments}

\noindent\textbf{Experimental setup} We evaluate the method proposed
in this paper empirically by addressing three case examples that demonstrate
the utility of our method. More specifically, we show that the
proposed \astrid method allows us to (1) identify attribute interactions
modelled by the classifier in a dataset, (2) generate (anonymised)
surrogate datasets with the same conditional distribution as an
original dataset, and (3) fuse datasets from different sources.

In the experiments we use the synthetic dataset
(Fig. \ref{fig:data:synthetic}) and 9 datasets from the UCI machine
learning repository \cite{bache:2014:a}\footnote{Datasets obtained
  from \url{http://www.cs.waikato.ac.nz/ml/weka/datasets.html}}. All
experiments were run in R \citep{R:2015:a} and the method is released
as the \astrid R-package. The \astrid R-package and the source code
for the experiments are available for
download\footnote{\url{https://github.com/bwrc/astrid-r}}.

All statistical significance testing in the experiments is conducted
at the $\alpha = 0.05$ level. We use a value of $R = 250$ and $R'=
100$ in Eqs.~\eqref{eq:empiricalp} and \eqref{eq:that},
respectively. In all experiments the dataset was randomly split as
follows: 50\% for training ($D_0$) and the rest for testing datasets
($D_{\mathrm{test}}$, see Eq. \eqref{eq:teststatistic}): 25\% for
computing of $\hat T$ (Eq. \eqref{eq:that}), and 25\% to find the
highest-cardinality grouping satisfying $p_\mathcal{S}\geq \alpha$
from among the results of the \astrid method. \\

\noindent\textbf{Classifiers} Classifier choice is important. The SVM
and random forest classifiers are among the best-performing
classifiers \citep{delgado:2014:a} and we hence use these (SVM with
RBF kernel from the \texttt{e1071} R-package \citep{e1071:2014:a} and
random forest from the \texttt{randomForest}
\citep{randomforest:2002:a} package). We also show some examples with
a \nb classifier (from \texttt{e1071}). All classifiers were used at
their default settings.  \\

\noindent\textbf{Datasets} The properties of the datasets are
summarised in Table~\ref{tab:datasets}, also showing the computation
time for the SVM and RF using unoptimised R-code. The computation time
of the \astrid method depends both on the properties of the dataset
such as the number of attributes and instances, and on the used
classifier. The UCI datasets were chosen to have at least 600 items
and so that the SVM and random forest classifiers achieve reasonably
good accuracy at default settings, since the goal here is to
demonstrate the applicability of the method rather than optimise
classifier performance. Rows with missing values and constant-value
columns were removed from the UCI datasets.

\begin{table}[ht]
\setlength{\tabcolsep}{1.2ex} 
  \centering
\caption{The datasets used in the experiments (2--10 from
  UCI). Columns as follows: Number of items (Ni) after removal of rows
  with missing values, number of classes (Nc) after removal of
  constant-value columns, number of attributes (Na). MCP is major
  class proportion. \textbf{T$_\mathrm{\textbf{SVM}}$} and
  \textbf{T$_\mathrm{\textbf{RF}}$} give the calculation in minutes of the \astrid method
  for the SVM and random forest, respectively.}
\label{tab:datasets}
\normalsize{
      \begin{tabularx}{\columnwidth}{clccccrr}
  \toprule
  \textbf{n} & \textbf{Dataset} & \textbf{Ni} & \textbf{Nc} & \textbf{Na} & \textbf{MCP} & \textbf{T$_\mathrm{\textbf{SVM}}$} & \textbf{T$_\mathrm{\textbf{RF}}$} \\ 
  \midrule
  1  & \texttt{synthetic} & 1000 & 2 & 4 & 0.50 & 0.4 & 1.1 \\
  2 & \texttt{balance-scale} & 625 & 3 & 4 & 0.46 & 0.4 & 1.3 \\ 
  3 & \texttt{diabetes} & 768 & 2 & 8 & 0.65 & 1.7 & 4.4 \\ 
  4 & \texttt{vowel} & 990 & 11 & 13 & 0.09 & 9.3 & 223.2 \\ 
  5 & \texttt{credit-a} & 653 & 2 & 15 & 0.55 & 5.6 & 10.0 \\ 
  6 & \texttt{segment} & 2310 & 7 & 18 & 0.14 & 24.8 & 45.3 \\ 
  7 & \texttt{vehicle} & 846 & 4 & 18 & 0.26 & 11.5 & 18.8 \\ 
  8 & \texttt{mushroom} & 5644 & 2 & 21 & 0.62 & 74.9 & 64.1 \\ 
  9 & \texttt{soybean} & 682 & 19 & 35 & 0.13 & 47.9 & 82.5 \\ 
  10 & \texttt{kr-vs-kp} & 3196 & 2 & 36 & 0.52 & 221.9 & 189.0 \\ 
  \bottomrule
      \end{tabularx}
}
\end{table}

\section{Results}
\label{sec:results}


\begin{table}[t!]
\centering
\caption{The \textrm{synthetic} dataset. The columns show the
  cardinality of the grouping ($k$), the average, minimum and maximum
  accuracy (acc) and its standard deviation (sd) from the random
  samples, and the $p$-value.}
\label{res:tab:synthetic}

\begin{subtable}{\columnwidth}
\centering
\caption{SVM}
\label{res:tab:synthetic:svm}
\small{
\begin{tabular}{cccccccccc} 
\toprule 
\textbf{k} & \textbf{acc$_\mathrm{ave}$} & \textbf{acc$_\mathrm{min}$} & \textbf{acc$_\mathrm{max}$} & \textbf{sd} & \textbf{p} & \rotatebox{90}{a4} & \rotatebox{90}{a3} & \rotatebox{90}{a2} & \rotatebox{90}{a1}\\ 
\cmidrule(r){1-6} 
\cmidrule(l){7-10} 
1 & 0.904 &  &  &  &  & (A & A & A & A)\\ 
2 & 0.904 & 0.876 & 0.924 & 0.009 & 0.614 & (A) & (B & B & B)\\ 
\rowcolor{cbluelight}
3 & 0.899 & 0.876 & 0.920 & 0.008 & 0.378 & (A) & (B) & (C & C)\\ 
4 & 0.738 & 0.652 & 0.824 & 0.031 & 0.004 & (A) & (B) & (C) & (D)\\ 
\bottomrule 
\end{tabular} 
}
\\ $ \mathcal{S} = \set{\set{1,2}, \set{3}, \set{4}}$
\end{subtable}

\vspace*{1em}
\begin{subtable}{\columnwidth}
\centering
\caption{Random forest}
\label{res:tab:synthetic:rf}
\small{
\begin{tabular}{cccccccccc} 
\toprule 
\textbf{k} & \textbf{acc$_\mathrm{ave}$} & \textbf{acc$_\mathrm{min}$} & \textbf{acc$_\mathrm{max}$} & \textbf{sd} & \textbf{p} & \rotatebox{90}{a4} & \rotatebox{90}{a3} & \rotatebox{90}{a1} & \rotatebox{90}{a2}\\ 
\cmidrule(r){1-6} 
\cmidrule(l){7-10} 
1 & 0.904 &  &  &  &  & (A & A & A & A)\\ 
2 & 0.902 & 0.880 & 0.928 & 0.008 & 0.486 & (A) & (B & B & B)\\ 
\rowcolor{cbluelight}
3 & 0.905 & 0.876 & 0.928 & 0.010 & 0.594 & (A) & (B) & (C & C)\\ 
4 & 0.720 & 0.652 & 0.788 & 0.027 & 0.004 & (A) & (B) & (C) & (D)\\ 
\bottomrule 
\end{tabular} \\
$\mathcal{S} = \set{\set{1,2}, \set{3}, \set{4}}$
}
\end{subtable}

\vspace*{1em}
\begin{subtable}{\columnwidth}
  \centering
\caption{Na\"ive Bayes}
\label{res:tab:synthetic:nb}
\small{
\begin{tabular}{cccccccccc} 
\toprule 
\textbf{k} & \textbf{acc$_\mathrm{ave}$} & \textbf{acc$_\mathrm{min}$} & \textbf{acc$_\mathrm{max}$} & \textbf{sd} & \textbf{p} & \rotatebox{90}{a1} & \rotatebox{90}{a2} & \rotatebox{90}{a3} & \rotatebox{90}{a4}\\ 
\cmidrule(r){1-6} 
\cmidrule(l){7-10} 
1 & 0.756 &  &  &  &  & (A & A & A & A)\\ 
2 & 0.756 & 0.756 & 0.756 & 0.000 & 1.000 & (A) & (B & B & B)\\ 
3 & 0.756 & 0.756 & 0.756 & 0.000 & 1.000 & (A) & (B) & (C & C)\\ 
\rowcolor{cbluelight}
4 & 0.756 & 0.756 & 0.756 & 0.000 & 1.000 & (A) & (B) & (C) & (D)\\ 
\bottomrule 
\end{tabular} \\
$\mathcal{S} = \set{\set{1}, \set{2}, \set{3}, \set{4}}$
}
\end{subtable}
\end{table}

\begin{table*} 
\centering 
\caption{Grouping of the \textrm{credit-a} dataset using SVM. Column headers as in Table~\ref{res:tab:synthetic}.}
\label{res:tab:credit-a:svm} 
\normalsize{
\begin{tabular}{ccccccccccccccccccccc}
\toprule 
\textbf{k} & \textbf{acc$_\mathrm{ave}$} & \textbf{acc$_\mathrm{min}$} & \textbf{acc$_\mathrm{max}$} & \textbf{sd} & \textbf{p} & \rotatebox{90}{A14} & \rotatebox{90}{A7} & \rotatebox{90}{A13} & \rotatebox{90}{A1} & \rotatebox{90}{A6} & \rotatebox{90}{A12} & \rotatebox{90}{A15} & \rotatebox{90}{A4} & \rotatebox{90}{A5} & \rotatebox{90}{A2} & \rotatebox{90}{A8} & \rotatebox{90}{A3} & \rotatebox{90}{A9} & \rotatebox{90}{A10} & \rotatebox{90}{A11}\\ 
\cmidrule(lr){1-6} 
\cmidrule(lr){7-21} 
1 & 0.871 &  &  &  &  & (A & A & A & A & A & A & A & A & A & A & A & A & A & A & A)\\ 
2 & 0.868 & 0.853 & 0.883 & 0.006 & 0.570 & (A & A & A) & (B & B & B & B & B & B & B & B & B & B & B & B)\\ 
3 & 0.866 & 0.847 & 0.883 & 0.006 & 0.371 & (A & A & A) & (B & B) & (C & C & C & C & C & C & C & C & C & C)\\ 
4 & 0.867 & 0.853 & 0.883 & 0.006 & 0.434 & (A & A & A) & (B & B) & (C) & (D & D & D & D & D & D & D & D & D)\\ 
5 & 0.867 & 0.847 & 0.883 & 0.007 & 0.426 & (A & A & A) & (B) & (C) & (D) & (E & E & E & E & E & E & E & E & E)\\ 
6 & 0.867 & 0.847 & 0.883 & 0.006 & 0.430 & (A & A) & (B) & (C) & (D) & (E) & (F & F & F & F & F & F & F & F & F)\\ 
7 & 0.867 & 0.847 & 0.883 & 0.007 & 0.478 & (A) & (B) & (C) & (D) & (E) & (F) & (G & G & G & G & G & G & G & G & G)\\ 
8 & 0.868 & 0.847 & 0.890 & 0.008 & 0.518 & (A) & (B) & (C) & (D) & (E) & (F) & (G & G & G & G) & (H & H & H & H & H)\\ 
9 & 0.868 & 0.847 & 0.883 & 0.007 & 0.514 & (A) & (B) & (C) & (D) & (E) & (F) & (G & G & G) & (H) & (I & I & I & I & I)\\ 
10 & 0.868 & 0.847 & 0.883 & 0.008 & 0.510 & (A) & (B) & (C) & (D) & (E) & (F) & (G) & (H & H) & (I) & (J & J & J & J & J)\\ 
\rowcolor{cbluelight}
11 & 0.861 & 0.834 & 0.883 & 0.010 & 0.267 & (A) & (B) & (C) & (D) & (E) & (F) & (G) & (H & H) & (I) & (J) & (K & K & K & K)\\ 
12 & 0.856 & 0.834 & 0.877 & 0.007 & 0.056 & (A) & (B) & (C) & (D) & (E) & (F) & (G) & (H) & (I) & (J) & (K) & (L & L & L & L)\\ 
13 & 0.847 & 0.822 & 0.877 & 0.009 & 0.024 & (A) & (B) & (C) & (D) & (E) & (F) & (G) & (H) & (I) & (J) & (K) & (L) & (M & M & M)\\ 
14 & 0.846 & 0.822 & 0.871 & 0.009 & 0.012 & (A) & (B) & (C) & (D) & (E) & (F) & (G) & (H) & (I) & (J) & (K) & (L) & (M) & (N & N)\\ 
15 & 0.847 & 0.816 & 0.877 & 0.012 & 0.028 & (A) & (B) & (C) & (D) & (E) & (F) & (G) & (H) & (I) & (J) & (K) & (L) & (M) & (N) & (O)\\ 
\bottomrule 
\end{tabular}
}
\flushleft{
  \hspace*{3em} $\mathcal{S}_{5} \phantom{i} = \set{\set{14, 7, 13}, \set{1}, \set{6}, \set{12}, \set{15, 4, 5, 2, 8, 3, 9, 10, 11}}$ \\
\hspace*{3em}  $\mathcal{S}_{11} = \set{\set{14}, \set{7}, \set{13}, \set{1}, \set{6}, \set{12}, \set{15}, \set{4, 5}, \set{2}, \set{8}, \set{3, 9, 10, 11}} $
  }
\end{table*}

\setlength{\tabcolsep}{0.6ex} 

\begin{table}[t!]
  \centering
\caption{Groupings for the synthetic and UCI datasets using SVM and
  random forest. The columns are as follows. The number of attributes
  in the dataset (\emph{N}), the size of the grouping (\emph{k}), the
  size of the largest ($N_1$) and second-largest ($N_2$) groups, and
  the significance of the grouping (\emph{p}). The other columns are:
  baseline accuracy when the classifier is trained with unshuffled
  data (a${}_0$) and with data shuffled using the found grouping
  (\emph{a}), the range of the accuracy (a${}_\textrm{range}$) and its
  standard deviation (a${}_\textrm{sd}$). The column $p_\textrm{OG}$
  is the $p$-value corresponding to Test 2 in \cite{ojala2010jmlr}.}

\label{res:tab:uci} 

  \begin{subtable}{\columnwidth}
    \centering
\caption{SVM}
\label{res:tab:uci:svm}
\small{
\begin{tabularx}{\columnwidth}{l  cccc rrr c rr}

    \toprule \textbf{Dataset}
    & \multicolumn{1}{c}{\textbf{N}} & \multicolumn{1}{c}{\textbf{k}}
    & \multicolumn{1}{c}{\textbf{N${}_1$}}
    & \multicolumn{1}{c}{\textbf{N$_2$}}
& \multicolumn{1}{c}{\textbf{p}}
    & \multicolumn{1}{c}{\textbf{a${}_0$}}
    & \multicolumn{1}{c}{\textbf{a}}
    & \multicolumn{1}{c}{\textbf{a${}_\textrm{range}$}}
    & \multicolumn{1}{c}{\textbf{a${}_\textrm{sd}$}}
    & \multicolumn{1}{c}{\textbf{p}$_\textrm{OG}$}\\ \midrule
\textbf{ balance-scale } &  4 & 3 & 2 & 1 & 0.12 & 0.89 & 0.86 & [0.78, 0.90] & 0.02 & 0.02\\
\textbf{ credit-a } &  15 & 12 & 4 & 1 & 0.06 & 0.87 & 0.86 & [0.83, 0.88] & 0.01 & 0.02\\
\textbf{ diabetes } &  8 & 8 & 1 & 1 & 0.59 & 0.71 & 0.71 & [0.66, 0.74] & 0.02 & 0.62\\
\textbf{ kr-vs-kp } &  36 & 33 & 4 & 1 & 1.00 & 0.92 & 0.92 & [0.92, 0.92] & 0.00 & 0.00\\
\textbf{ mushroom } &  21 & 15 & 7 & 1 & 0.14 & 1.00 & 0.99 & [0.99, 1.00] & 0.00 & 0.00\\
\textbf{ segment } &  18 & 4 & 15 & 1 & 0.05 & 0.95 & 0.94 & [0.92, 0.95] & 0.00 & 0.00\\
\textbf{ soybean } &  35 & 35 & 1 & 1 & 0.31 & 0.84 & 0.84 & [0.80, 0.86] & 0.01 & 0.29\\
\textbf{ vehicle } &  18 & 4 & 15 & 1 & 0.17 & 0.77 & 0.75 & [0.70, 0.79] & 0.02 & 0.00\\
\textbf{ vowel } &  13 & 3 & 11 & 1 & 0.06 & 0.81 & 0.78 & [0.74, 0.82] & 0.01 & 0.00\\
      \bottomrule
    \end{tabularx}
}
  \end{subtable}
  \vspace*{1em}\\
  \begin{subtable}{\columnwidth}
    \centering
\caption{Random forest}
\label{res:tab:uci:rf}

  \small{
  \begin{tabularx}{\columnwidth}{l cccc rrr c rr}
      \toprule
      \textbf{Dataset} 
& \multicolumn{1}{c}{\textbf{N}} 
& \multicolumn{1}{c}{\textbf{k}} 
& \multicolumn{1}{c}{\textbf{N${}_1$}} 
& \multicolumn{1}{c}{\textbf{N$_2$}} 
& \multicolumn{1}{c}{\textbf{p}} 
& \multicolumn{1}{c}{\textbf{a${}_0$}} 
& \multicolumn{1}{c}{\textbf{a}} 
& \multicolumn{1}{c}{\textbf{a${}_\textrm{range}$}} 
& \multicolumn{1}{c}{\textbf{a${}_\textrm{sd}$}} 
& \multicolumn{1}{c}{\textbf{p}$_\textrm{OG}$}\\
      \midrule
\textbf{ balance-scale } &  4 & 3 & 2 & 1 & 0.20 & 0.82 & 0.78 & [0.71, 0.84] & 0.03 & 0.01\\
\textbf{ credit-a } &  15 & 15 & 1 & 1 & 0.60 & 0.88 & 0.87 & [0.82, 0.91] & 0.01 & 0.18\\
\textbf{ diabetes } &  8 & 8 & 1 & 1 & 0.97 & 0.70 & 0.72 & [0.69, 0.75] & 0.01 & 0.90\\
\textbf{ kr-vs-kp } &  36 & 12 & 24 & 2 & 0.06 & 0.98 & 0.98 & [0.97, 0.98] & 0.00 & 0.00\\
\textbf{ mushroom } &  21 & 14 & 8 & 1 & 0.18 & 1.00 & 1.00 & [0.99, 1.00] & 0.00 & 0.00\\
\textbf{ segment } &  18 & 6 & 13 & 1 & 0.22 & 0.99 & 0.98 & [0.97, 0.99] & 0.00 & 0.00\\
\textbf{ soybean } &  35 & 23 & 11 & 3 & 0.05 & 0.96 & 0.95 & [0.93, 0.96] & 0.01 & 0.00\\
\textbf{ vehicle } &  18 & 6 & 12 & 2 & 0.09 & 0.75 & 0.72 & [0.69, 0.76] & 0.01 & 0.00\\
\textbf{ vowel } &  13 & 4 & 10 & 1 & 0.12 & 0.92 & 0.90 & [0.88, 0.93] & 0.01 & 0.00\\
      \bottomrule
      \end{tabularx}
  }
  \end{subtable}
\end{table}
\setlength{\tabcolsep}{6pt} 

We now present empirical results obtained using the \astrid method
and demonstrate the usefulness and applicability of the method in
three real-world contexts.

\subsection{Finding Attribute Interactions}
The results are presented as tables where columns represent
attributes, sorted in the order in which attributes were detached
in the sorting step. Each row represents a grouping. Attributes
belonging to the same group are marked with the same letter, i.e.,
attributes marked with the same letter on the same row are
interacting. The row with the highest-cardinality grouping for which
$p \geq 0.05$ is highlighted and this grouping is also shown below the
table.

Table~\ref{res:tab:synthetic} shows the results for the synthetic
dataset. The groupings of size $k=2$ and $k=3$ are valid ($p \geq
0.05$) for SVM and random forest. For $k=4$ it is clear that the
accuracy is lower for SVM and random forest, whereas for \nb also
$k=4$ is valid. The SVM and random forest classifiers both identify
the correct interaction structure of the dataset. In contrast, the \nb
classifier always assumes that each attribute is independent and all
groupings are equally valid as no interactions are utilised and hence
all groupings $k=1,\ldots,4$ yield the same accuracy. These results
mean that an SVM trained on the permuted synthetic dataset
using $\mathcal{S} = \set{\set{1,2}, \set{3}, \set{4}}$ is
indistinguishable (at the 5\% level) from an SVM trained on the
original synthetic dataset. We hence find the factorised form of the
joint distribution of the data.

Table~\ref{res:tab:credit-a:svm} shows the valid groupings of the
\texttt{credit-a} dataset using SVM. The results indicate that the SVM
classifier does not utilise many attribute interactions, and hence the
dataset can be permuted to a large extent without impacting classifier
performance. As an example, the grouping with $k = 11$ contains two
groups of interacting attribute, marked with \emph{H} and \emph{K},
respectively, while the other attributes are singletons. In variable
subset selection, as discussed in Example 3 in
Sec.~\ref{sec:introduction}, it is meaningful to pick attributes based
on their interaction. Here it would hence be meaningful to consider
the attributes in the \emph{H} and \emph{K} groups for $k = 11$ when
selecting features, since the attributes in these groups are likely
interacting.

The groupings for our datasets are summarised in
Table~\ref{res:tab:uci}, showing the properties in terms of the number
of attributes used by the classifiers and the number of groups and
singletons together with statistics describing the accuracy of the
classifier using the found grouping. The groupings for SVM and
random forest are in general similar in terms of the size of the found
grouping, $k$ in Table~\ref{res:tab:uci}, and the number of
singletons. However, there are some exceptions such as the
\texttt{kr-vs-kp} and \texttt{soybean} datasets. Also, the structures
utilised by the classifiers is somewhat different.

To compare our results with those of \cite{ojala2010jmlr},
investigating whether a classifier utilises attribute interactions, we
computed the $p$-value for their Test 2 (denoted $p_\mathrm{OG}$ in
Table~\ref{res:tab:uci}), which is equivalent to our Problem
\ref{prob:structuretest} with an all-singleton grouping.
$p_\mathrm{OG} \geq 0.05$ indicates that the classifier does not
utilise attribute interactions in the dataset. This occurs for
\texttt{diabetes} and \texttt{soybean} for SVM and for
\texttt{diabetes} and \texttt{credit-a} for random forest. This is in
line with our findings, since for these datasets $k$ equals $N$ in
Table~\ref{res:tab:uci} and no interactions are hence utilised.

\subsection{Anonymisation and Surrogate Data}
Data anonymisation (e.g., \cite{AgrawalS00,sweeney:2002:a,BayardoS03})
and synthetic data generation are related processes with the goal of
generating data that shares properties with the original data so that
the anonymised data can be used in place of the original
data. Anonymisation can be done in a principled way if we know the
attribute interactions in the dataset. One method of data
anonymisation is based on shuffling the data, which we employ here. As
a measure of anonymity we compute the proportion $P_\mathrm{anon}$ of
rows of the original dataset that are still intact after the shuffling.
Ideally the shuffled data has no rows that were present in
the original data.

We consider the \texttt{credit-a} dataset. It contains information
on credit card applications and has 15 attributes and two classes:
positive or negative decision. To illustrate how our method can be
applied to data anonymisation we here use the grouping for $k = 5$ in
Table~\ref{res:tab:credit-a:svm}, shown as $\mathcal{S}_{5}$ below the
table. The anonymised dataset is obtained from the original dataset by
permuting it with $\mathcal{S}_{5}$. The accuracy using original,
unshuffled data is $0.871$ and when the classifier is trained with
anonymised data the accuracy is $0.867$ (average of 100
repetitions). To test the quality of the anonymisation, we calculated
$P_\mathrm{anon}$ for 100 repetitions. The original training dataset
has 327 unique rows and we obtain an average $P_\mathrm{anon} =
0.1\%$, i.e., the anonymisation is very efficient,
and there is in practice no loss in classification accuracy.
Shuffling the data
using the \goldeneye permutation with $\mathcal{S}_{5}$ yields new
surrogate datasets with the same class-conditional distribution as the
original dataset.

\subsection{Efficient Data Fusion and Collection}
Assume that we have a small dataset and have computed its structure
using \astrid. Further assume that we want to collect more data to be
used as a training dataset for a classifier, and also assume that we
can assign class labels to data items from an external source.

An interesting implication of the method presented here is that it can
be used for efficient data fusion and collection. E.g., for the
\texttt{credit-a} dataset a data collection task can be split up into
the sets given by $\mathcal{S}_{5}$ in
Table~\ref{res:tab:credit-a:svm}. We may collect a new dataset by
letting one survey participant respond to only the attributes in some
of the sets in $\mathcal{S}_{5}$ and tag the answers with the class,
which for \texttt{credit-a} is the known credit
decision. Independently collected answers regarding different sets of
attributes can then be fused based on the class label and used to
train a new classifier. This works, because if the distribution of the
data obeys the class-conditional form of Eq.~\eqref{eq:jdfact}, then
it should not matter in training of a classifier if different
attribute groups on a row of the data matrix are in fact collected
from separate persons in the same class.

Collecting survey data is often costly and time-consuming as a large
number of people must answer a large number of questions. Subdividing
a survey therefore reduces answering time and costs. Using the above
described method we may hence speed up data collection for a training
dataset for a classifier.

In this manner new training data for a classifier can be collected
more efficiently, since it can be partitioned into independent sets
based on grouping of an initial dataset using \astrid.

\section{Discussion and Conclusion}
\label{sec:discussion}
We present an efficient framework for testing the hypothesis that the
class-conditional joint distribution of a dataset follows a specific
factorised form. The factorised joint data distribution tells us what
attribute interactions are used by a classifier and hence also what
the structure of the data is. Knowledge of the joint distribution is
important for data exploration and can be used in solving several
real-world data analysis problems, in this paper exemplified by
showing the utility in (i) finding attribute interactions
(applications to variable selection), (ii) data anonymisation and
generation of surrogate data, and (iii) data fusion. Our empirical
investigation shows that often many interactions in the datasets
that we considered can be broken substantially without affecting
classification performance.

The framework is realised in the \astrid method which is made
available as the \astrid
R-package\footnote{\url{https://github.com/bwrc/astrid-r}}. It
can be used to automatically find
the factorisation of highest granularity.
Both the framework and \astrid make no assumptions regarding the data
distribution or the used classifier and hence have high generic
applicability to different datasets and problems. The methods
presented here build upon and extend ideas discussed in
\cite{henelius:2014:peek, ojala2010jmlr}. However, unlike the
groupings found in, e.g., \cite{henelius:2014:peek}, the framework
presented here provides statistical guarantees.

When interpreting the results it is very important to also consider
the uniqueness and stability of the results, both of which depend on
several factors. The \astrid method is a randomised algorithm and the
found groupings are hence not necessarily unique. The stability of the
results depends on factors such as the size of the data and the
strength of the interactions in the data. This can be compared to,
e.g., clustering solutions found using $k$-means. Also, the results
are affected by the number of random samples ($R$ and $R'$ in
Eqs.~\eqref{eq:empiricalp} and \eqref{eq:that}), and for practical
applications a trade-off between accuracy and speed must be
made. However, it is important to pick a high enough number of random
samples as this, for instance, directly impacts the minimum obtainable
$p$-value. The \astrid method should be considered an explorative
technique for investigating the structure of the data.

The methods discussed in this paper rely on the assumption that the
classifier can learn the structure of the data. As an example, the
assumption of attribute independence of the \nb classifier means that
it is incapable of exploiting attribute interactions and all groupings
are hence equally valid, as exemplified in
Table~\ref{res:tab:synthetic}.

Furthermore, the relationship between statistical significance and
practical relevance should be considered. E.g., although a negligible
drop in performance due to factorisation of the joint distribution may
be statistically significant, such a decrease can be of little
practical relevance. Compare, e.g., the accuracies for $k=1$ and
$k=15$ in Table~\ref{res:tab:credit-a:svm}.

As an additional contribution and as a tool to solve our main problem,
we have introduced a novel clustering algorithm,
Alg.~\ref{listing:sid}, which is described in Sec.~\ref{sec:clustering}
which may have uses also in other contexts. The clustering algorithm
can be used if only the reward (or cost) of the clustering solution is
available and there is no natural distance measure between data
points, in which case the traditional clustering methods which use
inter-item distances are useless. Alg.~\ref{listing:sid} is heuristic
for a general clustering reward function, but as shown in
Thm.~\ref{thm:optimal}, the algorithm is optimal if the reward
function is monotonic, i.e., if the reward always decreases if one of
the ``true'' clusters is broken, which should intuitively be true
in many real situations.

A potential direction of future research is to extend the method to
regression problems, requiring a redefinition of the permutation
scheme so that the resulting datasets are samples from a factorised
distribution conditioned on the dependent variable of the regression
model. Also techniques to speed up the $p$-value computations by
reducing the number of random samples required are of
interest. Furthermore, the clustering problem
(Sec.~\ref{sec:clustering}) and the associated monotonicity
definition (Def.~\ref{def:mono}) may have applications also to other
problems.



\begin{acks}
This work was supported by Academy of Finland (decision 288814) and
Tekes (Revolution of Knowledge Work project).
\end{acks}

\bibliographystyle{abbrv}
\bibliography{structure_identification}

\end{document}